\documentclass{article}

\usepackage[final,nonatbib]{nips_2018}
\usepackage{color}
\usepackage{sgame, tikz}

\usepackage[utf8]{inputenc} 
\usepackage[T1]{fontenc}    
\usepackage{hyperref}       
\usepackage{url}            
\usepackage{booktabs}       
\usepackage{nicefrac}       

\usepackage{microtype, xcolor}
\usepackage{subfigure}
\usepackage{booktabs} 
\usepackage{hyperref}

\usepackage{url}
\usepackage{mathtools}
\mathtoolsset{
showonlyrefs=true 
}
\usepackage{amsmath, bbm, amsthm, amsfonts, amssymb, cite, enumerate}
\usepackage{cancel}
\usepackage{datetime,comment}

\definecolor{darkgreen}{RGB}{0,125,0}
\newcounter{mlNoteCounter}

\newcounter{dbaCounter}

\DeclareMathAlphabet\mathbfcal{OMS}{cmsy}{b}{n}

\DeclareMathOperator*{\im}{im}

\DeclareMathOperator*{\grad}{grad}
\DeclareMathOperator*{\rot}{rot}

\DeclareMathOperator*{\curl}{curl}

\DeclareMathOperator*{\logit}{logit}

\DeclareMathOperator*{\dva}{{{div}_a}}
\DeclareMathOperator*{\dve}{{{div}_e}}
\DeclareMathOperator*{\dv}{{div}}

\newcommand*\bLambda{\ensuremath{\boldsymbol\Lambda}}

\newcommand*\bOmega{\ensuremath{\boldsymbol\Omega}}

\newcommand{\bR}{{\mathbb R}}

\newcommand{\eval}{{\mathcal E}}

\newcommand{\ba}{{\mathbf a}}
\newcommand{\bbb}{{\mathbf b}}
\newcommand{\bc}{{\mathbf c}}
\newcommand{\bd}{{\mathbf d}}
\newcommand{\be}{{\mathbf e}}

\newcommand{\bn}{{\mathbf n}}
\newcommand{\bp}{{\mathbf p}}
\newcommand{\bP}{{\mathbf P}}
\newcommand{\bq}{{\mathbf q}}
\newcommand{\br}{{\mathbf r}}

\newcommand{\bs}{{\mathbf s}}
\newcommand{\bu}{{\mathbf u}}
\newcommand{\bv}{{\mathbf v}}
\newcommand{\wt}{{\mathbf w}}

\newcommand{\bA}{{\mathbf A}}
\newcommand{\bB}{{\mathbf B}}
\newcommand{\bC}{{\mathbf C}}
\newcommand{\bD}{{\mathbf D}}

\newcommand{\bQ}{{\mathbf Q}}
\newcommand{\bT}{{\mathbf T}}
\newcommand{\bU}{{\mathbf U}}
\newcommand{\bV}{{\mathbf V}}
\newcommand{\bS}{{\mathbf S}}
\newcommand{\bZ}{{\mathbf 0}}
\newcommand{\bO}{{\mathbf 1}}

\newtheorem{thm}{Theorem}
\newtheorem*{thm*}{Theorem}
\newtheorem{prop}{Proposition}
\newtheorem{prop*}{Proposition}
\newtheorem{lem}{Lemma}
\newtheorem{defn}{Definition}

\theoremstyle{remark}
\newtheorem{rem}{Remark}
\newtheorem{eg}{Example}

\title{Re-evaluating Evaluation}

\author{
  David Balduzzi\thanks{DeepMind. Email: \texttt{\{\,dbalduzzi\,|\,karltuyls\,|\,perolat\,|\,thore\,\}@google.com}}
   \And Karl Tuyls$^*$
   \And Julien Perolat$^*$
   \And Thore Graepel$^*$
  }

\begin{document}

\maketitle

\begin{abstract}
    \emph{``What we observe is not nature itself, but nature exposed to our method of questioning.''}
    -- Werner Heisenberg
    \vspace{1mm}
    
    Progress in machine learning is measured by careful evaluation on problems of outstanding common interest. However, the proliferation of benchmark suites and environments, adversarial attacks, and other complications has diluted the basic evaluation model by overwhelming researchers with choices. Deliberate or accidental cherry picking is increasingly likely, and designing well-balanced evaluation suites requires increasing effort.
    In this paper we take a step back and propose \emph{Nash averaging}.
    The approach builds on a detailed analysis of the \emph{algebraic structure of evaluation} in two basic scenarios: agent-vs-agent and agent-vs-task. 
    The key strength of Nash averaging is that it automatically adapts to redundancies in evaluation data, so that results are not biased by the incorporation of easy tasks or weak agents. Nash averaging thus encourages maximally inclusive evaluation -- since there is no harm (computational cost aside) from including all available tasks and agents. 
\end{abstract}

\section{Introduction}

Evaluation is a key driver of progress in machine learning, with e.g. ImageNet \cite{deng:09} and the Arcade Learning Environment \cite{BellemareNVB13} enabling  subsequent breakthroughs in supervised and reinforcement learning \cite{krizhevsky:12, mnih:15}. However, developing evaluations has received little \emph{systematic} attention compared to developing algorithms. Immense amounts of compute is continually expended smashing algorithms and tasks together -- \emph{but the results are almost never used to \textbf{evaluate and optimize evaluations}}. In a striking asymmetry, results are almost exclusively applied to evaluate and optimize algorithms.

The classic train-and-test paradigm on common datasets, which has served the community well \cite{donoho:15}, is reaching its limits. Three examples suffice. Adversarial attacks have complicated evaluation, raising questions about which attacks to test against \cite{szegedy:13, tramer:18, kurakin:18, uesato:18}. Training agents far beyond human performance with self-play means they can only really be evaluated against each other \cite{silver:17, silver:17a}. The desire to build increasingly general-purpose agents has led to a proliferation of environments: Mujoco, DM Lab, Open AI Gym, Psychlab and others \cite{todorov:12,  beattie:16, brockman:16, leibo:18}.

In this paper we pause to ask, and partially answer, some basic questions about evaluation: 
\textbf{Q1.} Do tasks test what we think they test? 
\textbf{Q2.} When is a task redundant?
\textbf{Q3.} Which tasks (and agents) matter the most?
\textbf{Q4.} How should evaluations be evaluated? 

We consider two scenarios: \emph{agent vs task} (AvT), where algorithms are evaluated on suites of datasets or environments; and \emph{agent vs agent} (AvA), where agents compete directly as in Go and Starcraft. Our goal is to treat tasks and agents symmetrically -- with a view towards, ultimately, co-optimizing agents and evaluations. From this perspective AvA, where the task is (beating) another agent, is especially interesting. Performance in AvA is often quantified using Elo ratings \cite{Elo78} or the closely related TrueSkill \cite{herbrich:07}. There are two main problems with Elo. Firstly, Elo bakes-in the assumption that relative skill is transitive; but Elo is meaningless -- it has no predictive power -- in cyclic games like rock-paper-scissors. Intransitivity has been linked to biodiversity in ecology, and may be useful when evolving populations of agents \cite{frean:01, kerr:02, laird:06, szolnoki:14}. Secondly, an agent's Elo rating can be inflated by instantiating many copies of an agent it beats (or conversely). This can cause problems when Elo guides hyper-optimization methods like population-based training \cite{jaderberg:17b}. Similarly, the most important decision when constructing a task-suite is which tasks to include. It is easy, and all too common, to bias task-suites in favor of particular agents or algorithms. 

\subsection{Overview}
Section~\ref{s:pre} presents background information on Elo and tools for working with antisymmetric matrices, such as the Schur decomposition and combinatorial Hodge theory. A major theme underlying the paper is that the fundamental algebraic structure of tournaments and evaluation is antisymmetric \cite{dgm:18}. Techniques specific to antisymmetric matrices are less familiar to the machine learning community than approaches like PCA that apply to symmetric matrices and are typically correlation-based. 

Section~\ref{s:alg} presents a unified approach to representing evaluation data, where agents and tasks are treated symmetrically. A basic application of the approach results in our first contribution: a \textbf{multidimensional Elo rating} (mElo) that handles cyclic interactions. We also sketch how the Schur decomposition can uncover latent skills and tasks, providing a partial answer to \textbf{Q1}. We illustrate mElo on the domain of training an AlphaGo agent \cite{DSilverHMGSDSAPL16}.

The second contribution of the paper is \textbf{Nash averaging}, an evaluation method that is \textbf{invariant} to redundant tasks and agents, see section~\ref{s:meta}. The basic idea is to play a meta-game on evaluation data \cite{lanctotl:17}. The meta-game has a unique maximum entropy Nash equilibrium. The key insight of the paper is that the maxent Nash adapts automatically to the presence of redundant tasks and agents. The maxent Nash distribution thus provides a principled answer to \textbf{Q2} and \textbf{Q3}: which tasks and agents do and do not matter is determined by a meta-game. Finally, expected difficulty of tasks under the Nash distribution on agents yields a partial answer to \textbf{Q4}. 
The paper concludes by taking a second look at the performance of agents on Atari. We find that, under Nash averaging, human performance ties with the best agents, suggesting better-than-human performance has not yet been achieved.

\subsection{Related work}
Legg and Hutter developed a definition of intelligence which, informally, states ``intelligence measures an agent's ability to achieve goals in a wide range of environments'' \cite{legg:05, legg:13}. Formally, they consider all computable tasks weighted by algorithmic complexity \cite{solomonoff:64, kolmogorov:65, chaitin:66}. Besides being incomputable, the distribution places (perhaps overly) heavy weight on the simplest tasks.

A comprehensive study of performance metrics for machine learning and AI can be found in \cite{ferri:09, hernandez:12, hernandez:17, hernandez:17a, Olson2017}. There is a long history of psychometric evaluation in humans, some of which has been applied in artificial intelligence \cite{spearman:04, woolley:10, bringsjord:11}. Bradley-Terry models provide a general framework for pairwise comparison \cite{hunter:04}. Researchers have recently taken a second look at the arcade learning environment \cite{BellemareNVB13} and introduced new performance metrics \cite{Machado17a}. However, the approach is quite particular. Recent work using agents to evaluate games has somewhat overlapping motivation with this paper \cite{liapis:13, horn:14, nielsen:15, silva:17, volz:18}. Item response theory is an alternative, and likely complementary, approach to using agents to evaluate tasks \cite{hambleton:91} that has recently been applied to study the performance of agents on the Arcade Learning Environment  \cite{martinez:17}. 

Our approach draws heavily on work applying combinatorial Hodge theory to statistical ranking \cite{jiang:11} and game theory \cite{candogan:11, candogan:13, candogan:13a}. We also draw on empirical game theory \cite{Walsh03, Wellman06}, by using a meta-game to ``evaluate evaluations'', see section~\ref{s:meta}. Empirical game theory has been applied to domains like poker and continuous double auctions, and has recently been extended to asymmetric games \cite{PhelpsPM04, PhelpsCMNPS07, PonsenTKR09, BloembergenTHK15, Tuyls18}. \emph{von Neumann winners} in the dueling bandit setting and NE-regret are related to Nash averaging \cite{dudik:15, balsubramani:16, Jordan:07, Jordan:10}.

\section{Preliminaries}
\label{s:pre}

\textbf{Notation.} Vectors are column vectors. $\bZ$ and $\bO$ denote the constant vectors of zeros and ones respectively. We sometimes use subscripts to indicate dimensions of vectors and matrices, e.g. $\br_{n\times1}$ or $\bS_{m\times n}$ and sometimes their entries, e.g. $\br_i$ or $\bS_{ij}$; no confusion should result. The unit vector with a 1 in coordinate $i$ is $\be_i$. Proofs and code are in the appendix.

\subsection{The Elo rating system}
\label{s:elo}

Suppose $n$ agents play a series of pairwise matches against each other. Elo assigns a rating $r_i$ to each player $i\in[n]$ based on their wins and losses, which we represent as an $n$-vector $\br$. The predicted probability of $i$ beating $j$ given their Elo ratings is 
\begin{equation}
    \label{eq:elo_probability}
    \hat{p}_{ij} := \frac{10^{r_i/400}}{10^{r_i/400} + 10^{r_j/400}}
    = \sigma (\alpha r_i - \alpha r_j ),
    \quad\text{where}\quad
    \sigma(x) = \frac{1}{1 + e^{-x}}
    \quad\text{and}\quad
    \alpha = \frac{\log(10)}{400}.
\end{equation}
The constant $\alpha$ is not important in what follows, so we pretend $\alpha=1$. Observe that only the \emph{difference} between Elo ratings affects win-loss predictions. We therefore impose that Elo ratings sum to zero, $\br^\intercal\bO=0$, without loss of generality.
Define the loss,
\begin{equation}
    \label{eq:elo_loss}
    \ell_\text{Elo}(p_{ij}, \hat{p}_{ij}) = -p_{ij} \log \hat{p}_{ij} - (1 - p_{ij})\log(1-\hat{p}_{ij}),
    \quad\text{where}\quad
    \hat{p}_{ij} = \sigma(r_i-r_j)
\end{equation}
and $p_{ij}$ is the true probability of $i$ beating $j$. Suppose the $t^\text{th}$ match pits player $i$ against $j$, with outcome $S_{ij}^t=1$ if $i$ wins and $S_{ij}^t=0$ if $i$ loses. 
Online gradient descent on $\ell_\text{Elo}$ obtains
\begin{equation}
    \label{eq:elo_update}
    r_i^{t+1}\leftarrow r_i^t - \eta\cdot \nabla_{r_i}\ell_\text{Elo}(S^t_{ij}, \hat{p}^t_{ij}) 
    = r_i^t + \eta\cdot(S^t_{ij} - \hat{p}^t_{ij}).
\end{equation}
Choosing learning rate $\eta=16$ or $32$ recovers the updates introduced by Arpad Elo in \cite{Elo78}.

The win-loss probabilities predicted by Elo ratings can fail in simple cases. For example, rock, paper and scissors will all receive the same Elo ratings. Elo's predictions are  $\hat{p}_{ij}=\frac{1}{2}$ for all $i,j$ -- and so Elo has no predictive power for any given pair of players (e.g. paper beats rock with $p=1$).

\textbf{What are the Elo update's fixed points?}
Suppose we batch matches to obtain empirical estimates of the probability of player $i$ beating $j$: $\bar{p}_{ij} = \sum_{n}\frac{S_{ij}^n}{N_{ij}}$. As the number of matches approaches infinity, the empirical estimates approach the true probabilities $p_{ij}$. 
\begin{prop}\label{prop:elo_fp}
    Elo ratings are at a stationary point under batch updates iff the matrices of empirical probabilities and predicted probabilities have the same row-sums (or, equivalently the same column-sums):
    \begin{equation}
    \label{eq:elo_fp}
        \nabla_{r_i}\left[ \sum_j\ell_\text{Elo}(\bar{p}_{ij},\hat{p}_{ij}) \right]= 0
        \;\forall i \quad\text{iff}\quad
        \sum_j \bar{p}_{ij} = \sum_j \hat{p}_{ij}\;\forall i.
    \end{equation}
\end{prop}

Many different win-loss probability matrices result in identical Elo ratings. The situation is analogous to how many different joint probability distributions can have the same marginals. We return to this topic in section~\ref{s:ava}.

\subsection{Antisymmetric matrices}
\label{s:anti}

We recall some basic facts about antisymmetric matrices. Matrix $\bA$ is antisymmetric if $\bA+\bA^\intercal=0$. Antisymmetric matrices have even rank and \emph{imaginary} eigenvalues  $\{\pm i\lambda_j\}_{j=1}^{\text{rank}(\bA)/2}$. Any antisymmetric matrix $\bA$ admits a real \textbf{Schur decomposition}:
\begin{equation}
  \bA_{n\times n} = \bQ_{n\times n}\cdot \bLambda_{n\times n}\cdot \bQ^\intercal_{n\times n},
\end{equation}
where $\bQ$ is orthogonal and $\bLambda$ consists of zeros except for $(2\times 2)$ diagonal-blocks of the form:
\begin{equation}
  \bLambda = \left(\begin{matrix}
    0 & \lambda_j \\
    -\lambda_j & 0
  \end{matrix}\right).
\end{equation}
The entries of $\bLambda$ are \emph{real} numbers, found by multiplying the eigenvalues of $\bA$ by $i=\sqrt{-1}$.

\begin{prop}\label{prop:schur}
  Given matrix $\bS_{m\times n}$ with rank $r$ and singular value decomposition $\bU_{m\times r} \bD_{r\times r}\bV_{r\times n}^\intercal$. Construct antisymmetric matrix 
    \begin{equation}
    \label{eq:anti_naive}
        \bA_{(m+n)\times(m+n)} = 
        \left(\begin{matrix}
            \bZ_{m\times m} & \bS_{m\times n} \\
            -\bS^\intercal_{n\times m} & \bZ_{n\times n}
        \end{matrix}\right).
    \end{equation}
  Then the thin Schur decomposition of $\bA$ is $\bQ_{(m+n)\times 2r}\bLambda_{2r\times 2r}\bQ_{2r\times (m+n)}^\intercal$ where the nonzero pairs in $\bLambda_{2r\times 2r}$ are $\pm$ the singular values in $\bD_{r\times r}$ and
  \begin{equation}
      \bQ_{(m+n)\times 2r} = \left(\begin{matrix}
          -\bu_1 & \bZ_{m\times1} & \cdots & -\bu_r & \bZ_{m\times1} & \\
          \bZ_{n\times1} & \bv_1 & \cdots & \bZ_{n\times1} & \bv_r
      \end{matrix}\right).
  \end{equation}
\end{prop}

\textbf{Combinatorial Hodge theory} 
is developed by analogy with differential geometry, see \cite{jiang:11, candogan:11, candogan:13, candogan:13a}. Consider a fully connected graph with vertex set $[n]=\{1,\ldots,n\}$. Assign a \textbf{flow} $\bA_{ij}\in\bR$ to each edge of the graph. The flow in the opposite direction $ji$ is $\bA_{ji} = -\bA_{ij}$, so flows are just $(n\times n)$ antisymmetric matrices. The flow on a graph is analogous to a vector field on a manifold. 

The combinatorial \textbf{gradient} of an $n$-vector $\br$ is the flow: $\grad(\br) := \br\bO^\intercal - \bO\br^\intercal$. Flow $\bA$ is a \textbf{gradient flow} if $\bA=\grad(\br)$ for some $\br$, or equivalently if $\bA_{ij} = \br_i - \br_j$ for all $i,j$.
The \textbf{divergence} of a flow is the $n$-vector $\dv(\bA) := \frac{1}{n}\bA\cdot \bO$. The divergence measures the contribution to the flow of each vertex, considered as a source. The \textbf{curl} of a flow is the three-tensor $\curl(\bA)_{ijk} = \bA_{ij} + \bA_{jk} - \bA_{ik}$. Finally, the \textbf{rotation} is $\rot(\bA)_{ij} = \frac{1}{n}\sum_{k=1}^n \curl(\bA)_{ijk}$.
\begin{thm*}[Hodge decomposition, \cite{jiang:11}]\label{thm:hodge}
    \textbf{(i)} $\dv\circ \grad(\br) = \br$ for any $\br$ satisfying $\br^\intercal\bO=0$.\\
    \textbf{(ii)} $\dv\circ\rot(\bA)=\bZ_{n\times 1}$ for any flow $\bA$. 
    \textbf{(iii)} $\rot\circ \grad(\br) = \bZ_{n\times n}$ for any vector $\br$.\\
    \textbf{(iv)} The vector space of antisymmetric matrices admits an orthogonal decomposition
    \begin{equation}
        \big\{\text{flows}\big\} = \big\{\text{antisymmetric matrices}\big\} 
        = \im(\grad) \oplus \im(\rot)
    \end{equation}
    with respect to the standard inner product $\langle\bA,\bB\rangle = \sum_{ij}\bA_{ij}\bB_{ij}$. Concretely, any antisymmetric matrix decomposes as
    \begin{equation}
        \bA = \big\{\text{transitive component}\big\} + 
        \big\{\text{cyclic component}\big\}
        = \grad(\br) + \rot(\bA)
        \quad\text{where}\quad\br = \dv(\bA).
    \end{equation}
\end{thm*}

\textbf{Sneak peak.}
The divergence recovers Elo ratings or just plain average performance depending on the scenario. The Hodge decomposition separates transitive  (captured by averages or Elo) from cyclic interactions (rock-paper-scissors), and explains when Elo ratings make sense. The Schur decomposition is a window into the latent skills and tasks not accounted for by Elo and averages.

\section{On the algebraic structure of evaluation}
\label{s:alg}

The Schur decomposition and combinatorial Hodge theory provide a unified framework for analyzing evaluation data in the AvA and AvT scenarios. In this section we provide some basic tools and present a multidimensional extension of Elo that handles cyclic interactions.

\subsection{Agents vs agents (AvA)}
\label{s:ava}

In AvA, results are collated into a matrix of win-loss probabilities based on relative frequencies. Construct  $\bA=\logit \bP$ with $\bA_{ij} := \log\frac{p_{ij}}{1-p_{ij}}$. Matrix $\bA$ is antisymmetric since $p_{ij} + p_{ji}=1$.

\textbf{When can Elo correctly predict win-loss probabilities?}
The answer is simple in logit space: 
\begin{prop}\label{prop:when_elo_works}
    \textbf{(i)} If probabilities $\bP$ are generated by Elo ratings $\br$ then the divergence of its logit is $\br$. That is,
    \begin{equation}
        \text{if }p_{ij}=\sigma(r_i-r_j)\,\,\forall i,j\,\text{ then }\dv(\logit\bP) = \Big(\frac{1}{n}\sum_{j=1}^n (r_i-r_j)\Big)_{i=1}^n
        =\br.
    \end{equation}
    \textbf{(ii)} There is an Elo rating that generates probabilities $\bP$ iff $\curl(\logit \bP)=0$. Equivalently, iff $\log \frac{p_{ij}}{p_{ji}} + \log \frac{p_{jk}}{p_{kj}} + \log \frac{p_{ki}}{p_{ik}} = 0$ for all $i,j,k$.
\end{prop}

Elo is, essentially, a uniform average in logit space. Elo's predictive failures are due to the cyclic component $\tilde{\bA}:= \rot(\logit \bP)$ that uniform averaging ignores.

\textbf{Multidimensional Elo (mElo$_{2k}$).}
Elo ratings bake-in the assumption that relative skill is transitive. However, there is no single dominant strategy in games like rock-paper-scissors or (arguably) StarCraft. Rating systems that can handle intransitive abilities are therefore necessary. An obvious approach is to learn a feature vector $\wt$ and a rating vector $\br_i$ per player, and predict $\hat{p}_{ij} = \sigma(\br_i^\intercal \wt - \br_j^\intercal\wt)$. Unfortunately, this reduces to the standard Elo rating since $\br_i^\intercal \wt$ is a scalar. 

Handling intransitive abilities requires learning an approximation to the cyclic component $\tilde{\bA}$. Combining the Schur and Hodge decompositions allows to construct low-rank approximations that extend Elo. Note, antisymmetric matrices have \emph{even} rank. Consider 
\begin{equation}
    \bA_{n\times n} = \grad(\br) + \tilde{\bA}\approx \grad(\br) + 
    \bC^\intercal\left(\begin{matrix}
        0 & 1 & \\
        -1 & 0 & \\
        & & \ddots
    \end{matrix}\right)
    \bC =: \grad(\br) + \bC_{n\times 2k}^\intercal\bOmega_{2k\times 2k}
     \bC_{2k\times n}
\end{equation}
where the rows of $\bC$ are orthogonal to each other, to $\br$, and to $\bO$. The larger $2k$, the better the approximation. Let \textbf{mElo$_{2k}$} assign each player Elo rating $r_i$ and $2k$-dimensional vector $\bc_i$. Vanilla Elo uses $2k=0$. The mElo$_{2k}$ win-loss prediction is
\begin{equation}
    \textbf{mElo$_{2k}$: }\, \hat{p}_{ij} = \sigma\Big(r_i - r_j +  \bc_i^\intercal\cdot \bOmega_{2k\times 2k} \cdot \bc_j\Big)
    \,\,\text{where}\,\,
    \bOmega_{2k\times 2k} = \sum_{i=1}^k(\be_{2i-1} \be_{2i}^\intercal - \be_{2i}\be_{2i-1}^\intercal).
\end{equation}
Online updates can be computed by gradient descent, see section~\ref{s:code}, with orthogonality enforced.

\subsection{Application: predicting win-loss probabilities in Go}
\label{s:alpha}

Elo ratings are widely used in Chess and Go. We compared the predictive capabilities of Elo and the simplest extension mElo$_2$ on eight Go algorithms taken from extended data table 9 in \cite{DSilverHMGSDSAPL16}: seven variants of AlphaGo, and Zen. The Frobenius norms and logistic losses are $\|\bP-\hat{\bP}\|_F = 0.85$ and $\ell_\text{log} = 1.41$ for Elo vs the empirical probabilities and $\|\bP-\hat{\bP}_2\|_F = 0.35$ and $\ell_\text{log} = 1.27$ for mElo$_2$.

To better understand the difference, we zoom in on three algorithms that were observed to interact non-transitively in \cite{Tuyls18}: $\alpha_v$ with value net, $\alpha_p$ with policy net, and Zen. Elo's win-loss predictions are poor (Table \textbf{Elo}: Elo incorrectly predicts both that $\alpha_p$ likely beats $\alpha_v$ and $\alpha_v$ likely beats Zen), whereas mElo$_2$ (Table \textbf{mElo$_2$}) correctly predicts likely winners in all cases (Table \textbf{empirical}), with more accurate probabilities:

\begin{center}
    \hspace{-1mm}
    \begin{tabular}{c|c c c|}
        \textbf{Elo} & $\alpha_v$ & $\alpha_p$ & Zen \\ 
        \hline
        $\alpha_v$ & - & 0.41 & 0.58 \\ 
        $\alpha_p$ & 0.59 & - & 0.67 \\ 
        Zen & 0.42 & 0.33 & - \\ 
        \hline
    \end{tabular}
    \quad
    \begin{tabular}{c|c c c|} 
        \textbf{empirical} & $\alpha_v$ & $\alpha_p$ & Zen \\ 
        \hline
        $\alpha_v$ & - & 0.7 & 0.4 \\ 
        $\alpha_p$ & 0.3 & - & 1.0 \\ 
        Zen & 0.6 & 0.0 & - \\ 
        \hline
    \end{tabular}
    \quad
    \begin{tabular}{c|c c c|} 
        \textbf{mElo$_2$} & $\alpha_v$ & $\alpha_p$ & Zen \\ 
        \hline
        $\alpha_v$ & - & 0.72 & 0.46 \\ 
        $\alpha_p$ & 0.28 & - & 0.98 \\ 
        Zen & 0.55 & 0.02 & - \\ 
        \hline
    \end{tabular}
\end{center}

\subsection{Agents vs tasks (AvT)}
\label{s:avt}

In AvT, results are represented as an $(m\times n)$ matrix $\bS$: rows are agents, columns are tasks, entries are scores (e.g. accuracy or total reward). Subtract the total mean, so the sum of all entries of $\bS$ is zero. We recast both agents and tasks as \emph{players} and construct an antisymmetric $(m+n)\times (m+n)$-matrix. 
Let $\bs = \frac{1}{m}\bS\cdot \bO_{m\times 1}$ and $\bd= -\frac{1}{n}\bS^\intercal\cdot\bO_{n\times 1}$ be the \textbf{average skill of each agent} and the \textbf{average difficulty of each task}. Define $\tilde{\bS} = \bS - (\bs\cdot\bO^\intercal - \bO\cdot \bd^\intercal)$. Let $\br$ be the concatenation of $\bs$ and $\bd$. We construct the  antisymmetric matrix
\begin{align}
    \label{eq:antisymmetrize}
    \bA_{(m+n)\times(m+n)} 
     = \grad(\br)
    + \underbrace{\left(\begin{matrix}
        \bZ_{m\times m} & \tilde{\bS}_{m\times n} \\
        -\tilde{\bS}^\intercal_{n\times m} & \bZ_{n\times n}
    \end{matrix}\right)}_{\tilde{\bA}}
     = \left(\begin{matrix}
        \grad(\bs) & \bS_{m\times n} \\
        -\bS^\intercal_{n\times m} & \grad(\bd)
    \end{matrix}\right).
\end{align}
The top-right block of $\bA$ is agent performance on tasks; the bottom-left is task difficulty for agents. The top-left block compares agents by their average skill on tasks; the bottom-right compares tasks by their average difficulty for agents. Average skill and difficulty explain the data if the score of agent $i$ on task $j$ is $\bS_{ij} = \bs_i - \bd_j$, the agent's skill minus the task's difficulty, for all $i,j$. Paralleling proposition~\ref{prop:when_elo_works}, averages explain the data, $\bS = \bs\bO^\intercal - \bO\bd^\intercal$, iff $\curl(\bA)=\bZ$.

The failure of averages to explain performance is encapsulated in $\tilde{\bS}$ and $\tilde{\bA}$. By proposition~\ref{prop:schur}, the SVD of $\tilde{\bS}$ and Schur decomposition of $\tilde{\bA}$ are equivalent. If the SVD is $\tilde{\bS}_{m\times n} = \bU_{m\times r}\bD_{r\times r}\bV^\intercal_{r\times n}$ then the rows of $\bU$ represent the latent abilities exhibited by agents and the rows of $\bV$ represent the latent problems posed by tasks.

\section{Invariant evaluation}
\label{s:meta}

Evaluation is often based on metrics like average performance or Elo ratings. Unfortunately, two (or two hundred) tasks or agents that look different may test/exhibit identical skills. Overrepresenting particular tasks or agents introduces biases into averages and Elo -- biases that can only be detected \emph{post hoc}. Humans must therefore decide which tasks or agents to retain, to prevent redundant agents or tasks from skewing results. At present, \emph{evaluation is not automatic and does not scale.}  To be scalable and automatic, an evaluation method should \emph{always benefit} from including additional agents and tasks. Moreover, it should \emph{adjust automatically and gracefully} to redundant data.

\begin{defn}
    An \textbf{evaluation method} maps data to a real-valued function on players (that is, agents or agents and tasks): 
    \begin{equation}
        \eval:\big\{\text{evaluation data}\big\} = \big\{\text{antisymmetric matrices}\big\} \rightarrow \Big[\big\{\text{players}\big\} \rightarrow \bR\Big]. 
    \end{equation}
\end{defn}

\textbf{Desired properties.}
    An evaluation method should be:
\begin{enumerate}[P1.]
    \item \emph{Invariant:} adding redundant copies of an agent or task to the data should make no difference. 
    \item \emph{Continuous:} the evaluation method should be robust to small changes in the data.
    \item \emph{Interpretable:} hard to formalize, but the procedure should agree with intuition in basic cases.
\end{enumerate}
Elo and uniform averaging over tasks are examples of evaluation methods that invariance excludes. 

\subsection{Nash averaging}
\label{s:mnmg}

This section presents an evaluation method satisfying properties $P1,P2,P3$. We discuss AvA here, see section~\ref{s:avt_meta} for AvT. Given antisymmetric logit matrix $\bA$, define a two-player meta-game with \emph{payoffs} $\mu_1(\bp,\bq) = \bp^\intercal \bA\bq$ and $\mu_2(\bp, \bq)=\bp^\intercal \bB\bq$ for the row and column meta-players, where $\bB=\bA^\intercal$. The game is symmetric because $\bB=\bA^\intercal$ and zero-sum because $\bB=-\bA$. 

The row and column meta-players pick `teams' of agents. Their payoff is the expected log-odds of their respective team winning under the joint distribution. If there is a dominant agent that has better than even odds of beating the rest, both players will pick it. In rock-paper-scissors, the only unbeatable-on-average team is the uniform distribution. In general, the value of the game is zero and the Nash equilibria are teams that are unbeatable in expectation. 

A problem with Nash equilibria (NE) is that they are not unique, which forces the user to make choices and undermines interpretability \cite{vonneumann:44, nash:50}. Fortunately, for zero-sum games there is a natural choice of Nash:

\begin{prop}[maxent NE]\label{prop:MaxNE}
    For antisymmetric $\bA$ there is a \emph{unique symmetric} Nash equilibrium $(\bp^*,\bp^*)$ solving $\max_{\bp\in\Delta_n} \min_{\bq\in\Delta_n} \bp^\intercal\bA\bq$ with greater entropy than any other Nash equilibrium.
\end{prop}
Maxent Nash is maximally indifferent between players with the same empirical performance. 
\begin{defn}
    The \textbf{maxent Nash evaluation method} for AvA is
    \begin{equation}
        \label{eq:eval}
        \eval_{m}: 
        \big\{\text{evaluation data}\big\} = \big\{\text{antisymmetric matrices}\big\}\xrightarrow{\text{maxent NE}} \left[\big\{\text{players}\big\}
        \xrightarrow{\text{Nash average}} \bR\right],
    \end{equation}
    where $\bp_\bA^*$ is the \textbf{maxent Nash equilibrium} and  $\bn_\bA:= \bA\cdot\bp^*_\bA$ is the \textbf{Nash average}.
\end{defn}
 Invariance to redundancy is best understood by looking at an example; for details see section~\ref{s:thm_main}.
\begin{eg}[invariance]
    Consider two logit matrices, where the second adds a redundant copy of agent $C$ to the first:
    \begin{center}
        \begin{tabular}{ c|c c c| } 
            $\bA$ & $A$ & $B$ & $C$ \\ 
            \hline
            $A$ & 0.0 & 4.6 & -4.6 \\ 
            $B$ & -4.6 & 0.0 & 4.6 \\ 
            $C$ & 4.6 & -4.6 & 0.0 \\ 
            \hline
        \end{tabular}
        \quad and\quad
        \begin{tabular}{ c|c c c c| } 
            $\bA'$ & $A$ & $B$ & $C_1$ & $C_2$ \\ 
            \hline
            $A$ & 0.0 & 4.6 & -4.6  & -4.6\\ 
            $B$ & -4.6 & 0.0 & 4.6 & 4.6 \\ 
            $C_1$ & 4.6 & -4.6 & 0.0 & 0.0\\ 
            $C_2$ & 4.6 & -4.6 & 0.0 & 0.0\\ 
            \hline
        \end{tabular}
    \end{center}
    The maxent Nash for $\bA$ is $\bp_\bA^*=(\frac{1}{3}, \frac{1}{3}, \frac{1}{3})$. It is easy to check that 
    $(\frac{1}{3}, \frac{1}{3}, \frac{\alpha}{3}, \frac{1-\alpha}{3})$ is Nash for $\bA'$ for any $\alpha\in[0,1]$ and thus the maxent Nash for $\bA'$ is
    $\bp^*_{\bA'}=(\frac{1}{3}, \frac{1}{3}, \frac{1}{6}, \frac{1}{6})$. Maxent Nash automatically detects the redundant agents $C_1,C_2$ and distributes $C$'s mass over them equally. 
    
    Uniform averaging is not invariant to adding redundant agents; concretely $\dv(\bA)=\bZ$ whereas $\dv(\bA')= (-1.15, 1.15,0,0)$, falsely suggesting agent $B$ is superior. In contrast, $\bn_\bA=\bZ_{3\times 1}$ and $\bn_{\bA'}=\bZ_{4\times1}$ (the zero-vectors have different sizes because there are different numbers of agents). Nash averaging correctly reports no agent is better than the rest in both cases.
\end{eg}
\vspace{-2mm}

\begin{figure}[t]
    \center
    \includegraphics[width=.99\textwidth]{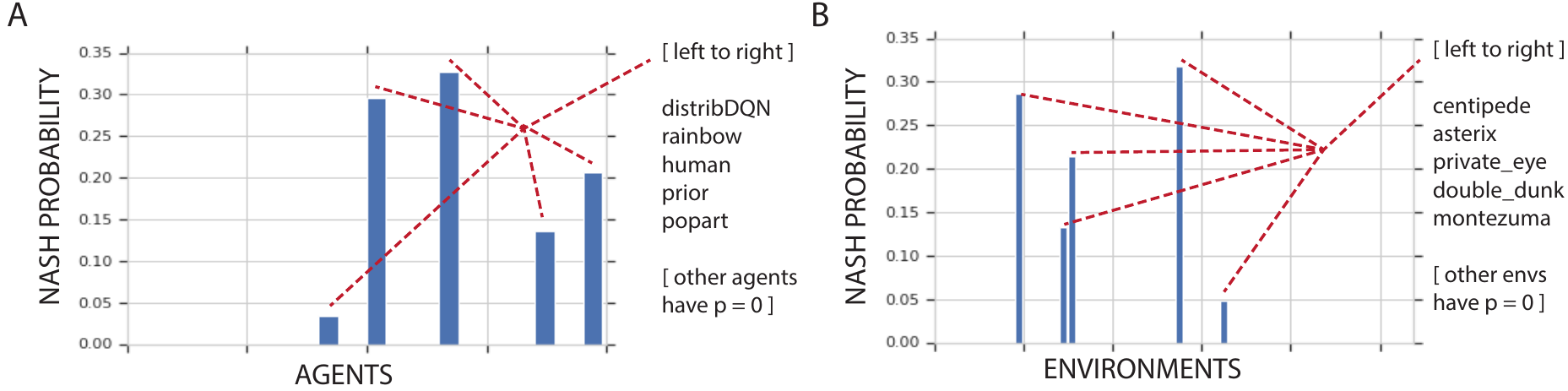}
    \vspace{-3mm}
    \caption{(A) The Nash $\bp_a^*$ assigned to agents; (B) the Nash $\bp_e^*$ assigned to environments. 
    }
    \label{f:nash_probs}
\end{figure}

\begin{thm}[main result for AvA\footnote{The main result for AvT is analogous, see section~\ref{s:avt_meta}. 
}]\label{thm:main}
    The maxent NE has the following properties:
    \vspace{-1mm}
    \begin{enumerate}[P1.]
        \item \textbf{Invariant:} Nash averaging, with respect to the maxent NE, is invariant to redundancies in $\bA$.
        \item \textbf{Continuous:} If $\bp^*$ is a Nash for $\hat{\bA}$ and $\epsilon = \|\bA-\hat{\bA}\|_\text{max}$ then $\bp^*$ is an $\epsilon$-Nash for $\bA$. 
        \item \textbf{Interpretable:}
        \textbf{(i)} The maxent NE on $\bA$ is the uniform distribution, $\bp^*=\frac{1}{n}\bO$, iff the meta-game is \emph{cyclic}, i.e. $\dv(\bA)=\bZ$.
        \textbf{(ii)} If the meta-game is \emph{transitive}, i.e. $\bA=\grad(\br)$, then the maxent NE is the uniform distribution on the player(s) with highest rating(s) -- there could be a tie.
    \end{enumerate}
\end{thm}
See section~\ref{s:thm_main} for proof and formal definitions. For interpretability, if $\bA=\grad(\br)$ then the transitive rating is all that matters: Nash averaging measures performance against the best player(s). If $\dv(\bA)=\bZ$ then no player is better than any other. Mixed cases cannot be described in closed form. 

The continuity property is quite weak:  theorem~\ref{thm:main}.2 shows the \emph{payoff} is continuous: \emph{a team that is unbeatable for $\hat{\bA}$ is $\epsilon$-beatable for nearby $\bA$}. Unfortunately, Nash equilibria themselves can jump discontinuously when $\bA$ is modified slightly. Perturbed best response converges to a more stable approximation to Nash \cite{hofbauer:02, sandholm:10} that unfortunately is not invariant.

\begin{eg}[continuity]
    Consider the cyclic and transitive logit matrices
    \begin{equation}
        \bC = \left(\begin{matrix}
            0 & 1 & -1 \\ -1 & 0 & 1 \\ 1 & -1 & 0
        \end{matrix}\right)
        \quad\text{and}\quad
        \bT = \left(\begin{matrix}
            0 & 1 & 2 \\ -1 & 0 & 1 \\ -2 & -1 & 0
        \end{matrix}\right).
    \end{equation}
    The maxent Nash equilibria and Nash averages of $\bC+\epsilon\bT$ are
    \begin{equation}
        \bp^*_{\bC+\epsilon\bT} = \begin{cases}
            \left(\frac{1+\epsilon}{3}, \frac{1-2\epsilon}{3}, \frac{1+\epsilon}{3}\right) & \text{if }0\leq \epsilon\leq \frac{1}{2}\\
            (1,0,0) & \text{if }\frac{1}{2}<\epsilon
        \end{cases}
        \quad\text{and}\quad
        \bn_{\bC+\epsilon\bT} = \begin{cases}
        (0,0,0) & 0\leq\epsilon\leq \frac{1}{2}\\ 
        (0,-1-\epsilon, 1-2\epsilon) & \frac{1}{2}<\epsilon
        \end{cases}
    \end{equation}
    The maxent Nash is the uniform distribution over agents in the cyclic case ($\epsilon=0$), and is concentrated on the first player when it dominates the others ($\epsilon>\frac{1}{2}$). When $0<\epsilon<\frac{1}{2}$ the optimal team has most mass on the first and last players. Nash jumps discontinuously at $\epsilon=\frac{1}{2}$.
\end{eg}

\subsection{Application: re-evaluation of agents on the Arcade Learning Environment}

To illustrate the method, we re-evaluate the performance of agents on Atari \cite{BellemareNVB13}. Data is taken from results published in \cite{wang:16, vanhasselt:16, ostrovski:17, hessel:17}. Agents include \texttt{rainbow}, \texttt{duel}ing networks, \texttt{prior}itized replay, \texttt{pop-art}, \texttt{DQN}, count-based exploration and baselines like \texttt{human}, \texttt{random}-action and \texttt{no-action}. The 20 agents evaluated on 54 environments are represented by matrix $\bS_{20\times 54}$. It is necessary to standardize units across environments with quite different reward structures: for each column we subtract the $\min$ and divide by the $\max$ so scores lie in $[0,1]$.

We introduce a meta-game where row meta-player picks aims to pick the best distribution $\bp^*_a$ on agents and column meta-player aims to pick the hardest distribution $\bp^*_e$ on environments, see section~\ref{s:avt_meta} for details. We find a Nash equilibrium using an LP-solver; it should be possible to find the maxent Nash using the algorithm in \cite{ortiz:06, ortiz:07}. The Nash distributions are shown in figure~\ref{f:nash_probs}. The supports of the distributions are the `core agents' and the `core environments' that form unexploitable teams. See appendix for tables containing all skills and difficulties.
panel B.

Figure~\ref{f:performance}A shows the skill of agents under uniform $\frac{1}{n}\bS\cdot\bO$ and Nash $\bS\cdot \bp_e^*$ averaging over environments; panel B shows the difficulty of environments under uniform $-\frac{1}{m}\bS^\intercal\cdot\bO$ and Nash $-\bS^\intercal\cdot\bp_a^*$ averaging over agents. There is a tie for top between the agents with non-zero mass -- including human. This follows by the indifference principle for Nash equilibria: strategies with support have equal payoff. 

Our results suggest that the better-than-human performance observed on the Arcade Learning Environment is because ALE is skewed towards environments that (current) agents do well on, and contains fewer environments testing skills specific to humans. Solving the meta-game automatically finds a distribution on environments that evens out the playing field and, simultaneously, identifies the most important agents and environments.

\begin{figure}[t]
    \center
    \includegraphics[width=.99\textwidth]{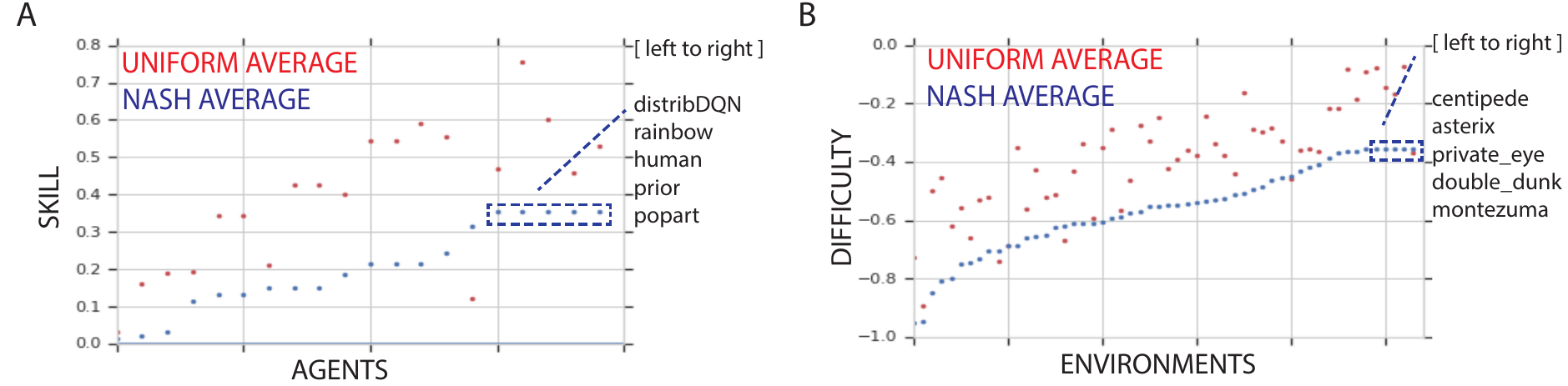}
    \vspace{-3mm}
    \caption{\textbf{Comparison of uniform and Nash averages.}
      (A) Skill of agents by uniform $\frac{1}{n}\bS\cdot\bO$ and Nash $\bS\cdot \bp_e^*$ averaging over environments. (B) Difficulty of environments under uniform $-\frac{1}{m}\bS^\intercal\cdot\bO$ and Nash $-\bS^\intercal\cdot\bp_a^*$ averaging over agents. Agents and environments are sorted by Nash-averages.
    }
    \label{f:performance}
\end{figure}

\section{Conclusion}

A powerful guiding principle when deciding what to measure is to find quantities that are \emph{invariant} to naturally occurring transformations. The determinant is computed over a basis -- however, the determinant is \emph{invariant to the choice of basis} since $\det(G^{-1}AG)=\det(A)$ for any invertible matrix $G$. Noether's theorem implies the dynamics of a physical system with symmetries obeys a conservation law.
The speed of light is fundamental because it is invariant to the choice of inertial reference frame.

One must have symmetries in mind to talk about invariance. \textbf{What are the naturally occurring symmetries in machine learning?} The question admits many answers depending on the context, see e.g. \cite{diaconis:88, lecun:98, kondor:03, kondor:08, zaheer:17, hartford:18, kondor:18}. In the context of evaluating agents, that are typically built from neural networks, it is unclear \emph{a priori} whether two seemingly different agents -- based on their parameters or hyperparameters -- are \emph{actually} different. Further, it is increasingly common that environments and tasks are parameterized -- or are learning agents in their own right, see self-play \cite{silver:17, silver:17a}, adversarial attacks \cite{szegedy:13, tramer:18, kurakin:18, uesato:18}, and automated curricula \cite{sukhbaatar:17}. The overwhelming source of symmetry when evaluating learning algorithms is therefore \textbf{redundancy}: different agents, networks, algorithms, environments and tasks that do basically the same job.

Nash evaluation computes a distribution on players (agents, or agents and tasks) that automatically adjusts to redundant data. It thus provides an invariant approach to measuring agent-agent and agent-environment interactions. In particular, Nash averaging encourages a maximally inclusive approach to evaluation: computational cost aside, the method should only benefit from including as many tasks and agents as possible. Easy tasks or poorly performing agents will not bias the results. As such Nash averaging is a significant step towards more \emph{objective} evaluation. 

Nash averaging is not always the right tool. Firstly, it is only as good as the data: \emph{garbage in, garbage out}. Nash decides which environments are important based on the agents provided to it, and conversely. As a result, the method is blind to differences between environments that do not make a difference to agents and vice versa. Nash-based evaluation is likely to be most effective when applied to a diverse array of agents and environments. Secondly, for good or ill, Nash averaging removes control from the user. One may have good reason to disagree with the distribution chosen by Nash. Finally, Nash is a harsh master. It takes an adversarial perspective and may not be the best approach to, say, constructing automated curricula -- although boosting is a related approach that works well \cite{freund:96, schapire:12}. It is an open question whether alternate invariant evaluations can be constructed, game-theoretically or otherwise.

\textbf{Acknowledgements.}
We thank Georg Ostrovski, Pedro Ortega, José Hernández-Orallo and Hado van Hasselt for useful feedback.
\setcounter{section}{0}
\renewcommand{\thesection}{\Alph{section}}
\vspace{5mm}
\noindent
{\textsf{\textbf{\Large{APPENDIX}}}}

\section{Invariance: Further motivation}
\label{s:inv_mot}

Consider agents A, B, C evaluated on a benchmark suite comprising three tasks:
\begin{center}
    \begin{tabular}{ c|c c c| c c|} 
        & task 1 & task 2 & task 3 & \textbf{average} & \textbf{rank} \\ 
        \hline
        agent A & 89 & 93 & 76 & 86 & $1^\text{st}$\\ 
        agent B & 85 & 85 & 85 & 85 & $2^\text{nd}$\\ 
        agent C & 79 & 74 & 99 & 84 & $3^\text{rd}$\\ 
        \hline
    \end{tabular}
\end{center}
On average, agent A performs best and agent C performs worst. Consider a second benchmark suite containing an additional fourth task. On the second benchmark suite, agent A performs worst and agent C performs best on average. However, a closer look at the second suite reveals that the additional task is a minor variant of one of the original three tasks. 
\begin{center}
    \begin{tabular}{ c|c c c c|c c|} 
        & task 1 & task 2 & task 3a & task 3b & \textbf{average} & \textbf{rank} \\ 
        \hline
        agent A & 89 & 93 & 76 & 77 & 84 & $3^\text{rd}$\\ 
        agent B & 85 & 85 & 85 & 84 & 85 & $2^\text{nd}$\\ 
        agent C & 79 & 74 & 99 & 98 & 88 & $1^\text{st}$\\ 
        \hline
    \end{tabular}
\end{center}
Measuring performance by uniformly averaging over tasks in a benchmark suite is sensitive to the set of tasks that are included in the suite. Including redundant tasks, whether consciously or not, can easily skew average performance in favor of or against particular agents. 

The problem becomes more serious when agent performance is measured against other agents, as in games such as Go, Chess or StarCraft. It is easy to manipulate the measured performance of agents by tweaking the composition of the population used to evaluate them. Consider the following example:
\begin{center}
    \begin{tabular}{ c|c c c| c|} 
        & agent A & agent B & agent C & \textbf{Elo}  \\ 
        \hline
        agent A & 0.5 & 0.9 & 0.1 & 0\\ 
        agent B & 0.1 & 0.5 & 0.9 & 0\\ 
        agent C & 0.9 & 0.1 & 0.5 & 0\\ 
        \hline
    \end{tabular}
\end{center}
The three agents exhibit rock-paper-scissors dynamics; their Elo ratings (normalized to sum to zero) are all zero. However, adding a second copy of agent C decreases the Elo rating of agent A and increases the Elo rating of agent B: 
\begin{center}
    \begin{tabular}{ c|c c c c| c|} 
        & agent A & agent B & agent C$_1$ & agent C$_2$ & \textbf{Elo}  \\ 
        \hline
        agent A     & 0.5 & 0.9 & 0.1 & 0.1 & -63\\ 
        agent B     & 0.1 & 0.5 & 0.9 & 0.9 &  63\\ 
        agent C$_1$ & 0.9 & 0.1 & 0.5 & 0.5 &   0\\ 
        agent C$_2$ & 0.9 & 0.1 & 0.5 & 0.5 &   0\\ 
        \hline
    \end{tabular}
\end{center}
That is, the Elo ratings of agents A and B are easily manipulated by changing the structure of the population. 

The examples above suggest it is important to find evaluation metrics that are \emph{invariant} to redundant changes in the population of agents or suite of tasks.

\paragraph{Related work.}
A different notion of measurement invariance has been proposed in the psychometric and consumer research literatures \cite{vandenberg:00}. There, measurement invariance refers to the statistical property that a measurement measures the same construct across a predefined set of groups. For example, whether or not a question in an IQ test is measurement invariant has to do with whether or not the question is interpreted in the same way by individuals with different cultural backgrounds.

\section{Proofs of propositions}
\label{s:proofs_of_props}

\textbf{Proof of proposition~\ref{prop:elo_fp}.}

\begin{prop*}
    Batch Elo updates are stable if the matrices of empirical probabilities and predicted probabilities have the same row-sums (or, equivalently the same column-sums):
    \begin{equation}
        \nabla_i \sum_j\ell_\text{Elo}(\bar{p}_{ij},\hat{p}_{ij}) = 0
        \;\forall i \quad\text{iff}\quad
        \sum_j \bar{p}_{ij} = \sum_j \hat{p}_{ij}\;\forall i.
    \end{equation}
\end{prop*}

\begin{proof}
    Player $i$'s weights are updated in one batch after observing win-loss probabilities $\bar{p}_{ij}$ for each player $j=1, \ldots, n$. Observe that
    \begin{equation}
        \nabla_i\left[ \sum_{j=1}^n\ell_\text{Elo}(\bar{p}_{ij},\hat{p}_{ij})\right]
        = \sum_{j=1}^n (\bar{p}_{ij} - \hat{p}_{ij}).
    \end{equation}
    The result follows.
\end{proof}

\textbf{Proof of proposition~\ref{prop:schur}.}
\begin{prop*}
  Given matrix $\bS_{m\times n}$ with rank $r$ and singular value decomposition $\bU \bD\bV^\intercal$. Construct antisymmetric matrix 
    \begin{equation}
        \bA_{(m+n)\times(m+n)} = 
        \left(\begin{matrix}
            \bZ_{m\times m} & \bS_{m\times n} \\
            -\bS^\intercal_{n\times m} & \bZ_{n\times n}
        \end{matrix}\right).
    \end{equation}
  Then the thin Schur decomposition of $\bA$ is $\bQ\bLambda\bQ^\intercal$ where the eigenpairs in the $\bLambda_{2r\times 2r}$ are $\pm$ the singular values in $\bD$ and
  \begin{equation}
      \bQ_{(m+n)\times r} = \left(\begin{matrix}
          -\bu_1 & \bZ_{m\times1} & \cdots & -\bu_r & \bZ_{m\times1} & \\
          \bZ_{n\times1} & \bv_1 & \cdots & \bZ_{n\times1} & \bv_r
      \end{matrix}\right)
  \end{equation}
\end{prop*}
\begin{proof}
    Direct computation; multiply out the matrices.
\end{proof}

\textbf{Proof of proposition~\ref{prop:when_elo_works}.}

\begin{prop*}
    \textbf{(i)} If probabilities $\bP$ are generated by Elo ratings $\br$ then the divergence of its logit is $\br$. That is,
    \begin{equation}
        \text{if }p_{ij}=\sigma(r_i-r_j)\,\,\forall i,j\,\text{ then }\dv(\logit\bP) = \Big(\frac{1}{n}\sum_{j=1}^n (r_i-r_j)\Big)_{i=1}^n
        =\br.
    \end{equation}
    \textbf{(ii)} There is an Elo rating that generates probabilities $\bP$ iff $\curl(\logit \bP)=0$. Alternatively, iff $\log \frac{p_{ij}}{p_{ji}} + \log \frac{p_{jk}}{p_{kj}} + \log \frac{p_{ki}}{p_{ik}} = 0$ for all $i,j,k$.
\end{prop*}

\begin{proof}
    For the first claim, apply definitions and recall $\br^\intercal\bO=0$. For the second, apply the Hodge decomposition from section~\ref{s:anti}.
\end{proof}

\textbf{Proof of proposition~\ref{prop:MaxNE}.}

\begin{prop*}[maxent NE]
    The game $\max_{\bp\in\Delta_n} \min_{\bq\in\Delta_n} \bp^\intercal\bA\bq$, where $\bA$ is antisymmetric, has a \emph{unique symmetric} Nash equilibrium $(\bp^*,\bp^*)$, with greater entropy than any other Nash equilibrium.
\end{prop*}

\begin{proof}
    The Nash equilibria in a two-player zero-sum are rectangular: if $(\bp,\bq)$ and $(\bu,\bv)$ are Nash equilibria then so are $(\bp, \bv)$ and $(\bu, \bq)$. Further, they form a convex polytope. Since $\bA^\intercal=-\bA$, the set of Nash equilibria is also symmetric: if $(\bp, \bq)$ is a Nash equilibrium then so is $(\bq, \bp)$. The entropy $H(\bp) := -\sum_{i=1}^n p_i \log p_i$ is strictly concave and therefore achieves a unique maximum on the compact, convex, symmetric set of Nash equilibria.
\end{proof}

\section{Proof of theorem~\ref{thm:main}}
\label{s:thm_main}

\begin{thm*}[main result for AvA]
    The maxent NE has the following properties:
    \vspace{-1mm}
    \begin{enumerate}[P1.]
        \item \textbf{Invariant:} Nash averaging, with respect to the maxent NE, is invariant to redundancies in $\bA$.
        \item \textbf{Continuous:} If $\bp^*$ is a Nash for $\hat{\bA}$ and $\epsilon = \|\bA-\hat{\bA}\|_\text{max}$ then $\bp^*$ is an $\epsilon$-Nash for $\bA$. 
        \item \textbf{Interpretable:}
        \textbf{(i)} The maxent NE on $\bA$ is the uniform distribution, $\bp^*=\frac{1}{n}\bO$, iff the meta-game is \emph{cyclic}, i.e. $\dv(\bA)=\bZ$.
        \textbf{(ii)} If the meta-game is \emph{transitive}, i.e. $\bA=\grad(\br)$, then the maxent NE is the uniform distribution on the player(s) with highest rating(s) -- there could be a tie.
    \end{enumerate}
\end{thm*}

\subsection{Proof of theorem~\ref{thm:main}.1}

First, we more precisely formalize invariance to redundancy.  

\begin{defn}[invariance to copying the last row and column]\label{def:redundancy}
    Given antisymmetric matrix $\bA_{n\times n}$, denote the right-most column by $\ba_n$. Assume the right-most column (and bottom row) of $\bA$ differs from all other columns. Construct antisymmetric matrix
    \begin{equation}
        \label{eq:actual_red}
        \bA'_{(n+1)\times (n+1)} = \left(\begin{matrix}
            \bA & \ba_n \\
            -\ba_n^\intercal & 0
            \end{matrix}\right)
    \end{equation}
    by adding an additional copy of the right-most column (and bottom row) to $\bA$.
    A family of functions
    \begin{equation}
        \Big\{\bp_k:\{\text{antisymmetric $k\times k$ matrices}\}\rightarrow \bR^k\Big\}_{k=1}^\infty
    \end{equation}
    is invariant to adding a row and column according to \eqref{eq:actual_red} if
    \begin{equation}
        \label{eq:actual_inv}
        \bp_{n+1}(\bA')^\intercal = \left(\bp_n(\bA)[1], \ldots, \bp_n(\bA)[n-1], \frac{\bp_n(\bA)[n]}{2}, \frac{\bp_n(\bA)[n]}{2}\right).
    \end{equation}
\end{defn}

If the copied row is not unique and receives positive mass under maxent Nash, then maxent Nash will already be spreading mass across the copies. In that case, adding \emph{yet another} copy will result in the maxent Nash on the larger mass spreading mass evenly across all copies. 

\begin{lem}\label{lem:copy_nash}
    Suppose $\bA_{n\times n}$ is antisymmetric with Nash equilibrium $\bp$. Construct $\bA'$ from $\bA$ by adding a redundant copy of the right-most column and  bottom row according \eqref{eq:actual_red}. Then 
    \begin{equation}
        \bp_\alpha' = \big(p_1, \ldots, p_{n-1},\alpha\cdot p_n, (1-\alpha)\cdot p_n\big)
    \end{equation}
    is a Nash equilibrium for $\bA'$ for all $\alpha\in[0,1]$. Conversely, if $\bp'$ is a Nash equilibrium for $\bA'$ then
    \begin{equation}
        (p'_1, \ldots, p'_{n-1}, p'_n + p'_{n+1})
    \end{equation}
    is a Nash equilibrium for $\bA$.
\end{lem}
\begin{proof}
    Since the value of the game on $\bA$ is zero, it follows that $\bp$ is a Nash equilibrium iff all the coordinates of $\bA\bp$ are nonnegative, i.e. $\bA\bp\succeq \bZ_{n\times 1}$. Direct computation shows that $\bA'\bp'_\alpha\succeq\bZ_{(n+1)\times 1}$, which implies $\bp'_\alpha$ is a Nash equilibrium for $\bA'$. The converse follows similarly.
\end{proof}

Finally, we prove theorem~\ref{thm:main}.2.

\begin{proof}
    The definition and lemma are stated in the particular case where the last column and row are copied. This is simply for notational convenience. They generalize trivially to copying arbitrary row/columns into arbitrary positions, and can be applied inductively to the cases where $\bA'$ is constructed from $\bA$ by inserting multiple redundant copies.
\end{proof}

\subsection{Proof of theorem~\ref{thm:main}.2}

Recall the max-norm on matrices is $\|\bS\|_\text{max}:= \max_{ij}|\bS_{ij}|$. 

\begin{defn}[$\epsilon$-Nash equilibrium]
    A joint strategy $(\bp^*,\bq^*)$ is an $\epsilon$-Nash equilibrium for $\bA$ if the benefit from either player deviating, separately, is at most $\epsilon$:
    \begin{equation}
        \max_{\bp'} (\bp'-\bp^*)^\intercal\bA\bq^*\leq \epsilon
        \quad\text{and}\quad
        \max_{\bq'} (\bp^*)^\intercal\bA(\bq^* - \bq')\leq \epsilon.
    \end{equation}
\end{defn}

We are now ready to prove
\begin{enumerate}[P2]
    \item \textbf{Continuous:} If $\bp^*$ is a Nash for $\hat{\bA}$ and $\epsilon = \|\bA-\hat{\bA}\|_\text{max}$ then $\bp^*$ is an $\epsilon$-Nash for $\bA$. 
\end{enumerate}

\begin{proof}
    Suppose $(\bp^*,\bp^*)$ is a Nash equilibrium for the antisymmetric matrix $\hat{\bA}$. Observe that
    \begin{equation}
        (\bp'-\bp^*)^\intercal \bA \bp^*  = (\bp')^\intercal \hat{\bA} \bp^* + (\bp')^\intercal \bA \bp^* - (\bp')^\intercal \hat{\bA} \bp^*
    \end{equation}
    for any distribution $\bp'$ because $(\bp^*)^\intercal\bA\bp^*=0$ since $\bA$ is antisymmetric. It follows that 
    \begin{equation}
        \max_{\bp'} \Big\{(\bp'-\bp^*)^\intercal \bA \bp^*\Big\}
        \leq \max_{\bp'}\Big\{(\bp')^\intercal \hat{\bA} \bp^*\Big\} +
        \max_{\bp'}\Big\{(\bp')^\intercal( \bA -  \hat{\bA}) \bp^*\Big\}
    \end{equation}
    The first term on the right-hand-side is $\leq0$ since $\bp^*$ is a Nash equilibrium for $\hat{\bA}$ and the value of the game is zero. The second term on the right-hand-side is $\leq \epsilon$ because $\|\bA-\hat{\bA}\|_\text{max}\leq \epsilon$ and $\bp'$ and $\bp^*$ are probability distributions.
\end{proof}

Note that since $\|\bS\|_\text{max}\leq \|\bS\|_2\leq \|\bS\|_F$ for any $\bS$, the divergence from Nash is also controlled by how well $\hat{\bA}$ approximates $\bA$ in the operator or Frobenius norms.

The proof is adapted from the proof of the following lemma in \cite{Tuyls18}. 
\begin{lem}[Nash on approximate games]\label{lem:approx_nash}
    Suppose $(\bp^*,\bq^*)$ is a Nash equilibrium for $\hat{\bA}$ and that $\epsilon = \|\bA-\hat{\bA}\|_\text{max}$. Then $(\bp^*, \bq^*)$ is a $2\epsilon$-Nash for $\bA$. 
\end{lem}
Our result is slightly sharper because we specialize to antisymmetric matrices.

\subsection{Proof of theorem~\ref{thm:main}.3}
\begin{proof}
    \textbf{(i)}
    If $\dv(\bA)= \bZ$ then $\frac{1}{n}\bO^\intercal \bA = \bZ^\intercal$ and $\bA\frac{1}{n}\bO = \bZ$ implying the uniform distribution is a Nash equilibrium because there is no incentive to deviate from $(\frac{1}{n} \bO,\frac{1}{n}\bO)$. The uniform distribution also has maximum entropy. Conversely, suppose $\dv(\bA)\neq \bZ$. Then $\bO^\intercal\bA$ has at least one positive and one negative coordinate because we know $\bO^\intercal\bA\bO=0$ by antisymmetry. It follows that if the row player chooses $\frac{1}{n}\bO^\intercal$ then the column player is incentivized to choose a distribution with more mass on the positive coordinate and less on the negative. In other words, the column player will not play the uniform distribution, and the uniform distribution is therefore not a Nash equilibrium.
    
    \textbf{(ii)}
    By assumption $\bA=\grad(\br) = \br\bO^\intercal - \bO\br^\intercal$, so $\bp^\intercal\bA\bq = \bp^\intercal\br - \br^\intercal\bq$ decouples into an independent maximization problem with respect to $\bp$ and minimization problem with respect to $\bq$. It follows that the optimal distribution $\bp^*$ concentrates mass on the maximal coordinate(s) of $\br$ and so does $\bq^*$. That is, to be a Nash equilibrium, $\bp^*$ and $\bq^*$ must place their mass on the maximal coordinate if it is unique and can distribute it arbitrarily over the set of maximal coordinates if there is a tie. Adding the condition that the Nash equilibrium has maximum entropy entails placing the uniform distribution over the maximal coordinate(s).
\end{proof}

\section{Nash averaging for agent-vs-task}
\label{s:avt_meta}

Given score matrix $\bS_{m\times n}$, construct antisymmetric matrix 
\begin{align}
    \label{eq:naive_antisymmetrize}
    \bA_{(m+n)\times(m+n)} 
     = \left(\begin{matrix}
        \bZ_{m\times m} & \bS_{m\times n} \\
        -\bS^\intercal_{n\times m} & \bZ_{n\times n}
    \end{matrix}\right).
\end{align}
Note this differs from the antisymmetrization used in section~\ref{s:avt}, see next remark. 
\begin{rem}
    The graph structure underlying AvT is \emph{bipartite}: agents interact with tasks and tasks with agents, but there are no direct agent-agent or task-task interactions. When done in full generality, the definitions of $\dv$, $\grad$ and $\curl$ take into account the graph structure, see \cite{jiang:11}. In particular, $\dv$, $\grad$ and $\curl$ are computed differently on bipartite graphs than fully connected graphs -- note the definitions in section~\ref{s:anti} are specific to fully connected graphs. Working in full generality is overkill for our purposes. It suffices to introduce slightly \emph{ad hoc} notation to handle the specific case of AvT.
\end{rem}

Introduce the notation
\begin{equation}
    \dva(\bS) =\frac{1}{m}\bS\cdot \bO_{m\times 1} 
    \quad\text{and}\quad
    \dve(\bS) = -\frac{1}{n}\bS^\intercal \cdot\bO_{n\times 1},
\end{equation}
where $\dva(\bS)$ measures uniform \textbf{average skill of agents} on tasks and $\dve(\bS)$ measures uniform \textbf{average difficulty of tasks} for agents. Let
\begin{equation}
    \grad(\bs,\bd) = \bs\cdot\bO^\intercal - \bO\cdot \bd^\intercal.
\end{equation}
Define a two-player zero-sum meta-game
\begin{equation}
    \max_{(\bp_a, \bp_e)\in\Delta_m\times\Delta_n}
    \min_{(\bq_a, \bq_e)\in\Delta_m\times\Delta_n}
    \big(\bp_a;\bp_e\big)^\intercal \bA\big(\bq_a;\bq_e\big).
\end{equation}
The setup is the same as for AvA in the main text except the row and column meta-players each play two distributions: one on agents and one on tasks/environments. The same argument as in proposition~\ref{prop:MaxNE} shows there is a unique symmetric maxent Nash equilibrium $\Big((\bp_a^*, \bp_e^*), (\bp_a^*, \bp_e^*)\Big)$.
\begin{defn}
    The \textbf{maxent Nash evaluation method} for AvT is
    \begin{equation}
        \eval_{m}: 
        \big\{\text{evaluation data}\big\} = \big\{\text{antisymmetric matrices}\big\}\xrightarrow{\text{maxent NEs}} \left[\big\{\text{players}\big\}
        \xrightarrow{\text{Nash averages}} \bR\right],
    \end{equation}
    where $\bp_a^*$ and $\bp_e^*$ are the \textbf{maxent Nash equilibria} over agents and environments and  $\bn_e:= -\bS^\intercal\cdot\bp^*_a$ and $\bn_a:= \bS\cdot\bp^*_e$ are the \textbf{Nash averages} quantifying difficulty of environments (Nash averaged over agents) and skill of agents (Nash averaged over environments) respectively.
\end{defn}
Suppose $\bS$ decomposes as
\begin{equation}
    \bS =  \grad(\bs,\bd) + \tilde{\bS},
\end{equation}
where $\dva(\tilde{\bS}) = \bZ$ and $\dve(\tilde{\bS}) = \bZ$. Observe that 
\begin{equation}
    \big(\bp_a;\bp_e\big)^\intercal \bA\big(\bq_a;\bq_e\big) = 
    \bs^\intercal(\bp_a-\bq_a) + \bd^\intercal(\bp_e - \bq_e) + \bp_a^\intercal\tilde{\bS} \bq_e + \bq_a^\intercal \tilde{\bS} \bp_e
\end{equation}
If $\tilde{\bS} \equiv \bZ_{m\times n}$ then the game reduces to
\begin{equation}
    \big(\bp_a;\bp_e\big)^\intercal \bA\big(\bq_a;\bq_e\big) = 
    \bs^\intercal(\bp_a-\bq_a) + \bd^\intercal(\bp_e - \bq_e) 
\end{equation}
and so the row player maximizes its payoff by putting all its agent-mass on the most skillful agent(s) and all of its environment-mass on the most difficult task(s) -- and similarly for the column player.

It is easy to check that the maxent Nash $(\bp_a^*, \bp_e^*)$ has $\bp_a$ a uniform distribution on all agents iff $\dva(\bS) =\bZ$, and has $\bp_e$ a uniform distribution on tasks iff $\dve(\bS)=\bZ$. We thus obtain

\begin{thm}[main result for AvT]
    The maxent NE has the following properties:
    \vspace{-1mm}
    \begin{enumerate}[P1.]
        \item \textbf{Invariant:} Nash averaging is invariant to redundancies in $\bA$.
        \item \textbf{Continuous:} If $(\bp_a^*, \bp_e^*)$ is a Nash for $\hat{\bA}$ and $\frac{\epsilon}{2} = \|\bA-\hat{\bA}\|_\text{max}$ then $(\bp_a^*, \bp_e^*)$ is an $\epsilon$-Nash for $\bA$. 
        \item \textbf{Interpretable:}
        \textbf{(i)} The agent component of the maxent NE on $\bA$ is the uniform distribution on agents, $\bp_a^*=\frac{1}{n}\bO$, iff  $\dva(\bA)=\bZ$.\\
        \textbf{(ii)} The task component of the maxent NE on $\bA$ is the uniform distribution on tasks, $\bp_e^*=\frac{1}{m}\bO$, iff  $\dve(\bA)=\bZ$.\\
        \textbf{(iii)} If the meta-game is \emph{transitive}, i.e. $\bA=\grad(\bs, \bd)$, then the maxent NE is the uniform distribution on the most skillful agent(s) and the uniform distribution on the most difficult task(s) -- there could be ties.
    \end{enumerate}
\end{thm}

\section{Code for computing mElo$_2$ updates}
\label{s:code}

The routine $\mathtt{mElo2\_update}$ takes as input: a pair of players $i,j$, the probability $\mathtt{p\_ij}$ of player $\mathtt{i}$ beating player $\mathtt{j}$ (which could be 0 or 1 if only a single match is observed on the given round), the rating vector $\mathtt{r}$ and the $n\times2$ matrix $\mathtt{c}$ quantifying non-transitive interactions. It returns updates to the $\mathtt{i^{th}}$ and $\mathtt{j^{th}}$ entries of $\mathtt{r}$ and $\mathtt{c}$.

$\mathtt{\texttt{def mElo2\_update}(i,\, j,\, p\_ij,\, r,\, c):}$\\
$\texttt{}\quad\mathtt{p\_hat\_ij = sigmoid(r[i] - r[j] + c[i,0] * c[j,1] - c[j,0] * c[i,1])}$\\
$\texttt{}\quad\mathtt{delta = p\_ij - p\_hat\_ij}$\\
$\texttt{}\quad\mathtt{r\_update = [\,16 * delta,\, -16 * delta\,]}$\\
$\texttt{}\quad\quad\text{\# \texttt{r} has higher learning rate than \texttt{c}}$\\
$\texttt{}\quad\mathtt{c\_update = [}$\\
$\texttt{}\quad\quad\mathtt{[\, +delta * c[j,1],\, -delta * c[i,1]\,],}$\\
$\texttt{}\quad\quad\mathtt{[\,-delta * c[j,0],\,  + delta * c[i,0]\,]}$\\
$\texttt{}\quad\mathtt{]}$\\
$\texttt{}\quad\mathtt{return\,\, r\_update,\, c\_update}$\\

\section{On the geometry of antisymmetric matrices}
\label{s:geom_melo}
 
 This section provides some intuition for multidimensional Elo ratings by describing some of the underling geometry.

\begin{figure}[t]
    \center
    \includegraphics[width=.99\textwidth]{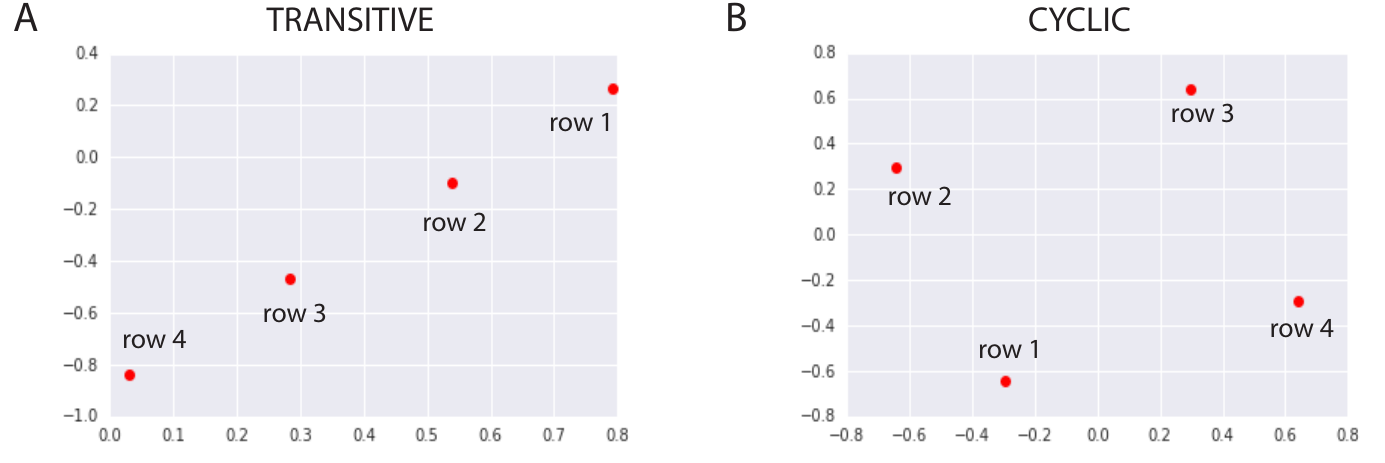}
    \vspace{-3mm}
    \caption{\textbf{Visualizing Schur decompositions.} (A) Rows of $\bQ^\bT_{4\times 2}$ form a straight line, reflecting the transitive structure of $\bT$. (B): Rows of $\bQ^\bC_{4\times 2}$ lie on a circle centered at the origin.
    }
    \label{f:vis_schur}
\end{figure}

\subsection{Visualizing the Schur decomposition}

Consider the following logit matrices:
    \begin{center}
        \begin{tabular}{ c|c c c c| } 
            $\bT$ & $A$ & $B$ & $C$ & $D$ \\ 
            \hline
            $A$ & 0 & 1 & 2  & 3\\ 
            $B$ & -1 & 0 & 1 & 2 \\ 
            $C$ & -2 & -1 & 0 & 1\\ 
            $D$ & -3 & -2 & -1 & 0\\ 
            \hline
        \end{tabular}
        \quad\text{and}\quad
        \begin{tabular}{ c|c c c c| } 
            $\bC$ & $A$ & $B$ & $C$ & $D$ \\ 
            \hline
            $A$ & 0 & 1 & 0  & -1\\ 
            $B$ & -1 & 0 & 1 & 0 \\ 
            $C$ & 0 & -1 & 0 & 1\\ 
            $D$ & 1 & 0 & -1 & 0\\ 
            \hline
        \end{tabular}
    \end{center}
Note that $\bT$ is transitive since $\bT = \grad\circ \dv(\bT)$ and $\bC$ is cyclic since $\dv(\bC)=\bZ$. Both matrices have rank two, so the thin Schur decomposition can be written
\begin{equation}
    \bQ_{4\times 2}\bLambda_{2\times 2}\bQ^\intercal_{2\times4}
\end{equation}
in either case. The matrices $\bQ^\bT_{4\times 2}$ and $\bQ^\bC_{4\times 2}$ arising from the respective Schur decompositions 
\begin{equation}
    \bQ^\bT_{4\times 2} = \left(\begin{matrix}
        0.793 &  0.267 \\
        0.538 & -0.101 \\
        0.284 & -0.469 \\
        0.029 & -0.836
    \end{matrix}\right)
    \quad\text{and}\quad
    \bQ^\bC_{4\times 2} = \left(\begin{matrix}
        -0.296 & -0.642 \\
        -0.642 &  0.296 \\
        0.296 &  0.642\\
         0.642 & -0.296
    \end{matrix}\right)
\end{equation}
can separately be thought of (and plotted) as four two-vectors, see figure~\ref{f:vis_schur}. The rows of $\bQ^\bT_{4\times 2}$ form a straight line, see panel A. More generally, if $\bA=\grad(\br)$ and $\br$ is not identically zero then $\bA$ is rank two with thin Schur decomposition $\bA=\bQ_{n\times 2}\bLambda_{2\times2}\bQ^\intercal_{2\times n}$, where the rows of $\bQ$ are points along a line. More precisely, the rows are of the form
\begin{equation}
    \bq_{k} = \ba + \alpha_k \bbb
    \quad\text{for }k\in\{1,\ldots, n\}
\end{equation}
where $\ba,\bbb\in\bR^2$ and $\alpha_k\in\bR$. 

The rows of $\bQ^\bC_{4\times 2}$ lie on a circle centered at the origin, see panel B. This does not generalize to arbitrary cyclic matrices. However, the geometry of antisymmetric matrices has interesting connections with areas of parallelograms and the complex plane, see next subsection.

\begin{figure}[t]
    \center
    \includegraphics[width=.99\textwidth]{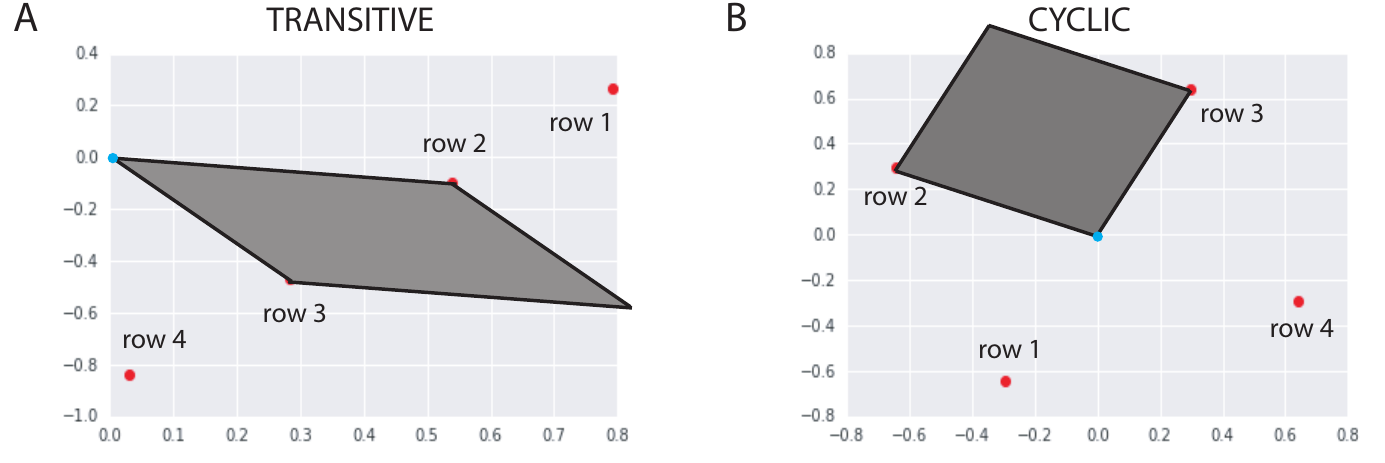}
    \vspace{-3mm}
    \caption{\textbf{Visualizing logits.} The entry $\bA_{ij}$ of $\bA$ is $\lambda$ times the \emph{signed area} of the parallelogram covered by the origin $(0,0)$, $\bQ_{i, \bullet}$, $\bQ_{i, \bullet}+\bQ_{j,\bullet}$ and $\bQ_{j,\bullet}$, where $\bQ_{i, \bullet}$ and $\bQ_{j, \bullet}$ are vectors corresponding to row $i$ and row $j$ of $\bQ_{4\times 2}$. 
    }
    \label{f:vis_logits}
\end{figure}

\subsection{Areas and phases}
Observe that
\begin{equation}
    \label{eq:signed_area}
    \left(\begin{matrix}u_1 & u_2\end{matrix}\right)
    \left(\begin{matrix}0 & \lambda \\ -\lambda & 0\end{matrix}\right)
    \left(\begin{matrix}v_1\\v_2\end{matrix}\right)
    = \lambda\cdot (u_1 v_2 - u_2 v_1),
\end{equation}
which is $\lambda$ times the signed area of the parallelogram covered by $(0,0)$, $\bu$, $\bu+\bv$, and $\bv$. It follows that the Schur decomposition breaks $\bR^n$ into an orthogonal collection of two-dimensional spaces -- one for each $\pm\lambda_i$ block -- and that the entries $\bA_{ij}$ of $\bA$ are sums of signed areas of parallelograms, one  parallelogram per two-dimensional subspace. The case where there is a single two-dimensional subspace (since the rank of $\bA$ is two) is illustrated in figure~\ref{f:vis_logits}.

Alternatively, introduce complex numbers $w =u_1 + iu_2$ and $z = v_1 + i v_2$, where $w = |w|\cdot e^{i\cdot \phi_w}$ and $z = |z|\cdot e^{i\cdot \phi_z}$ in polar coordinates. Then,  \eqref{eq:signed_area} can be rewritten
\begin{equation}
    \left(\begin{matrix}u_1 & u_2\end{matrix}\right)
    \left(\begin{matrix}0 & \lambda \\ -\lambda & 0\end{matrix}\right) = \text{Im}(w\cdot \bar{z}) = |w|\cdot |z| \sin(\phi_w-\phi_z).
\end{equation}

\subsection{Comment on Schur and Hodge}

The Schur decomposition is not guaranteed to be compatible with the Hodge decomposition. That is, $\grad(\bA)$ is not necessarily the span of two rows of the matrix $\bQ_{n\times r}$ arising in the Schur decomposition. Roughly, this happens when $\tilde{\bA}:= \bA-\grad\circ\dv(\bA)$ is not a good $(\text{rank}_\bA-2)$-approximation to $\bA$ in the $\|\cdot\|_2$-norm.

We recommend to first extract the transitive component and then perform the Schur decomposition on $\tilde{\bA}=\rot(\bA)$. Although this may not always be optimal with respect to the $\|\cdot\|_2$-norm, it has the important advantage that $\dv(\bA)$ is readily understood by humans as a measure of average performance.

\begin{figure}[t]
    \center
    \includegraphics[width=.85\textwidth]{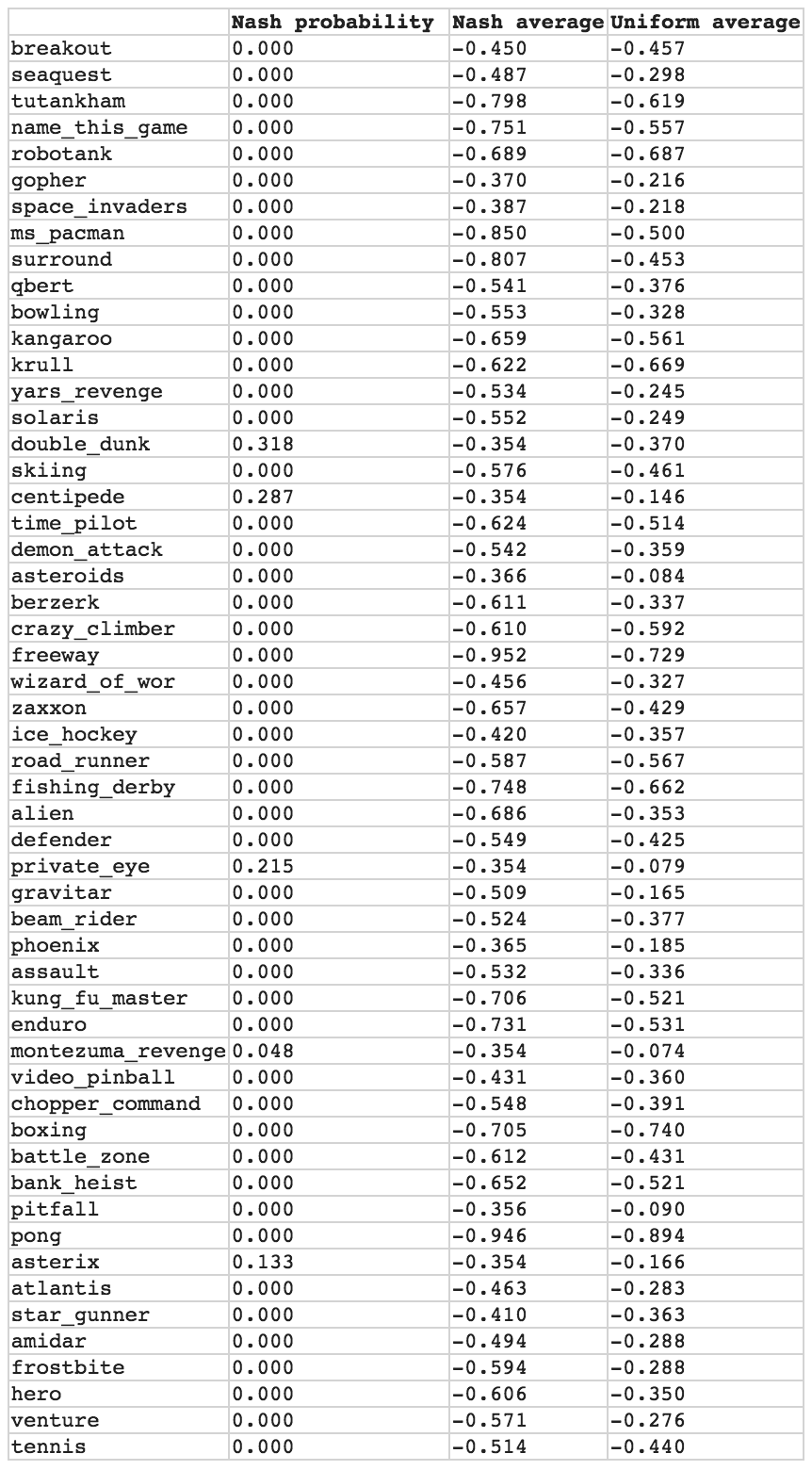}
    \vspace{-3mm}
    \caption{Evaluation of environments.
    }
    \label{f:envs}
\end{figure}

\begin{figure}[t]
    \center
    \includegraphics[width=.85\textwidth]{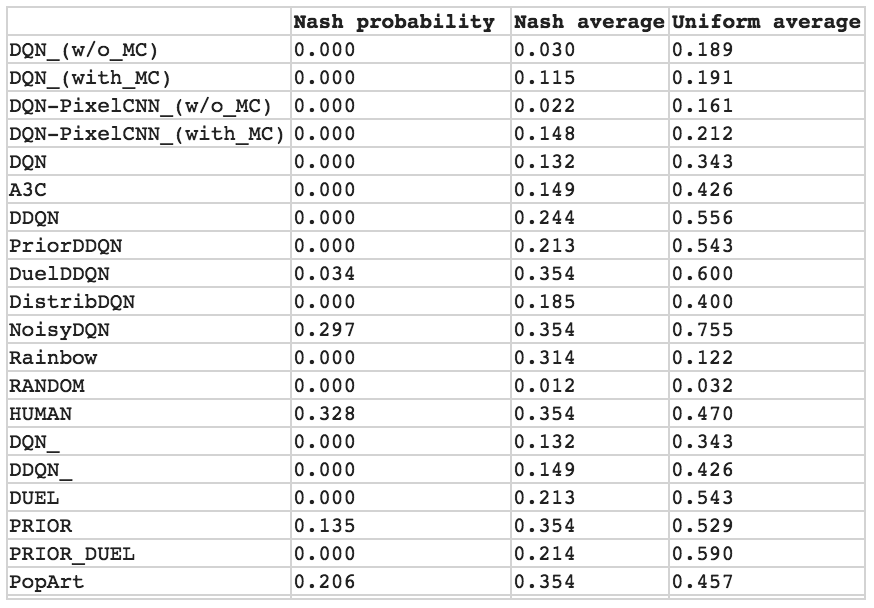}
    \vspace{-3mm}
    \caption{Evaluation of agents. Note, there are redundancies since agents are taken from multiple papers; these are ignored by Nash averaging.
    }
    \label{f:agents}
\end{figure}

\end{document}